\documentclass[11pt]{article}

\usepackage[preprint]{acl}

\usepackage{times}
\usepackage{latexsym}
\usepackage{amsmath}
\usepackage{booktabs}

\usepackage[T1]{fontenc}

\usepackage[utf8]{inputenc}

\usepackage{microtype}

\usepackage{inconsolata}

\usepackage{graphicx}

\title{Reliable Agent Memory through Fine-Grained Transactions}

\author{%
Han Xiao$^{1}$, Zhiyong Bao$^{2}$, Hongjun Xu$^{1}$, Xin Zhang$^{1}$, \\
\textbf{Yidong Chen$^{1,*}$, Xiaodong Shi$^{1,*}$} \\
{\normalfont $^{1}$Xiamen University \quad $^{2}$Northeastern University} \\
{\small\normalfont\texttt{ydchen@xmu.edu.cn}, \texttt{mandel@xmu.edu.cn}}}

\begin{document}
\maketitle

\begin{abstract}
Persistent memory enables agents to remember user preferences and provide personalized assistance over extended interactions, but a single incorrect update can affect future responses and decisions. Reliable updates must correct outdated information without treating unreliable or unverified claims as facts. Existing methods use memory quality or downstream task performance as learning signals, but direct supervision comparing what alternative operations preserve, overwrite, or discard remains limited. For example, rejecting an unreliable claim and deferring an unverified claim can both leave accepted information unchanged, yet deferral retains the claim pending further evidence. We introduce MemTransit, a benchmark with annotations specifying the appropriate action and memory contents after each update. It covers five executable actions (append, revise, ignore, reject, and defer) and evaluates both action selection and execution outcomes. We further propose MemTact, which learns memory updates through counterfactual execution. During training, it executes alternative actions given the same initial memory and incoming information, comparing their results with the annotated correct outcome to supervise action selection. Experiments on MemTransit show that MemTact improves action prediction and the accuracy of updated memory, reduces memory pollution, better preserves conflicting information, and limits error accumulation across successive updates.
\end{abstract}

\begin{figure}[t]
\centering
\includegraphics[
width=\columnwidth,
height=0.5\textheight,
keepaspectratio
]{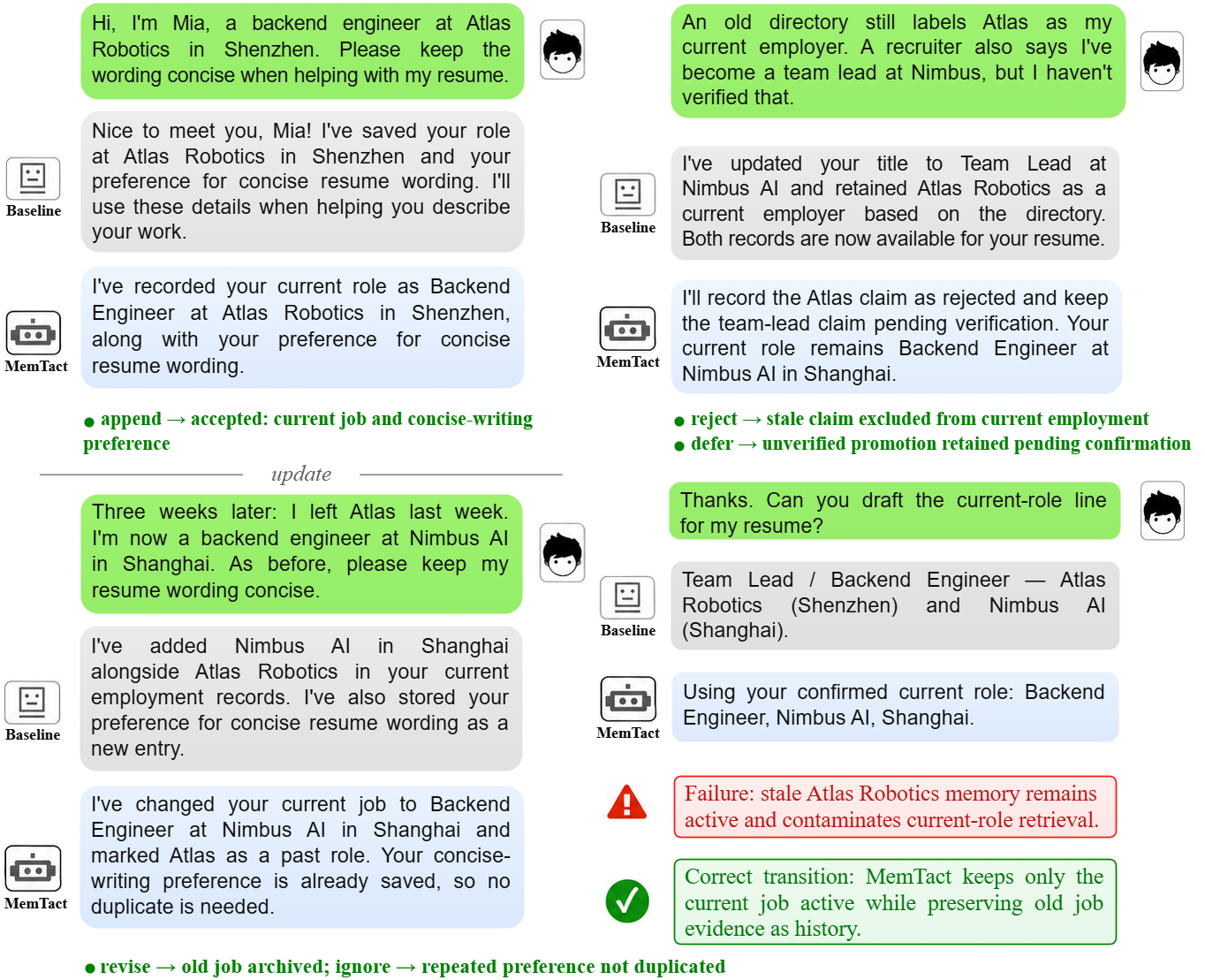}
\caption{
Example of reliable memory updating: MemTact uses five actions to keep memory current and prevent stale or unverified information from contaminating subsequent answers.
}
\label{fig:motivation}
\end{figure}

\section{Introduction}

Persistent memory enables agents to remember user preferences, track changes in tasks, and provide consistent assistance over extended interactions \citep{park2023generative,zhong2023memorybank,packer2023memgpt,xu2025amem}. However, memory also allows errors to persist: once outdated records or unverified claims are treated as facts, they can repeatedly affect subsequent responses and decisions \citep{chen2025halumem}. Reliable memory updates must therefore account for how incoming information relates to existing records: which facts have changed, which remain valid, and which claims lack sufficient support to be accepted.

As shown in Figure~1, after Mia changes jobs, the record of her previous position should be retained as history. When an unconfirmed promotion is reported, the record of her current position should remain unchanged, while the new claim is retained pending verification. Accepting the claim directly would introduce unverified information, whereas ignoring it would discard information awaiting verification. Reliable updates must therefore determine both what information to retain and what status to assign it.

Existing memory systems already support adding, revising, and deleting information \citep{chhikara2025mem0}, and recent methods learn update policies through downstream task performance, memory quality, or process rewards \citep{yan2025memoryr1,yu2026agemem,xu2026memcoe}. However, direct supervision comparing the specific changes caused by different operations remains limited. For example, discarding the promotion claim and retaining it pending verification may receive the same question-answering feedback if both yield the same answer about the current position, despite preserving different information. Action labels alone also cannot fully establish whether an update is correct, because they specify the operation type without identifying the record it should affect. Even when a model correctly selects revise, an error in target identification can cause it to modify unrelated information. We therefore ask: Given the same initial memory and incoming information, how can comparing what different actions preserve, modify, or discard help models learn correct memory updates?

To this end, we first introduce \textbf{MemTransit}, a benchmark that annotates both the correct update action and the resulting memory contents. Built from HaluMem \citep{chen2025halumem}, LoCoMo \citep{maharana2024locomo}, and LongMemEval \citep{wu2024longmemeval}, it covers five executable actions (append, revise, ignore, reject, and defer) and distinguishes accepted, pending, and rejected information. Each example includes annotations for the target record, action, and next state, allowing evaluation of both the selected operation and the information preserved, modified, or lost after execution. In successive updates, the memory produced at each step becomes the input to the next, making it possible to trace how an error affects subsequent updates and whether errors accumulate.

Building on this benchmark, we propose \textbf{MemTact}, which learns update decisions through counterfactual execution. During training, MemTact executes different actions given the same initial memory and incoming information, compares the resulting memory states with the annotated correct state, and uses these differences to supervise action selection. The model thus learns not only action labels but also the actual effects of different actions on memory contents. Annotated correct states are used only during training; at inference time, updates are based on the current memory and incoming information.

Experiments on MemTransit show that MemTact improves action prediction and the accuracy of updated memory, reduces memory pollution, better preserves conflicting evidence for subsequent verification, and limits error accumulation across successive updates. Our main contributions are:
\begin{enumerate}
    \item \textbf{Memory update formulation.}
    We distinguish accepted, pending, and rejected information and specify how incoming information changes existing memory through five executable actions, allowing the retention of a claim and its acceptance as fact to be represented and evaluated separately.

    \item \textbf{MemTransit benchmark.}
    We jointly annotate update actions, target records, and next states to evaluate both action selection and the resulting memory contents, and examine error accumulation through successive updates.

    \item \textbf{MemTact method.}
    We execute different actions given the same initial memory and incoming information, and compare the resulting memory contents with the annotated outcome to supervise action selection directly through the consequences of execution.
\end{enumerate}

\begin{figure*}[t]
\centering
\includegraphics[
width=1\textwidth,
keepaspectratio
]{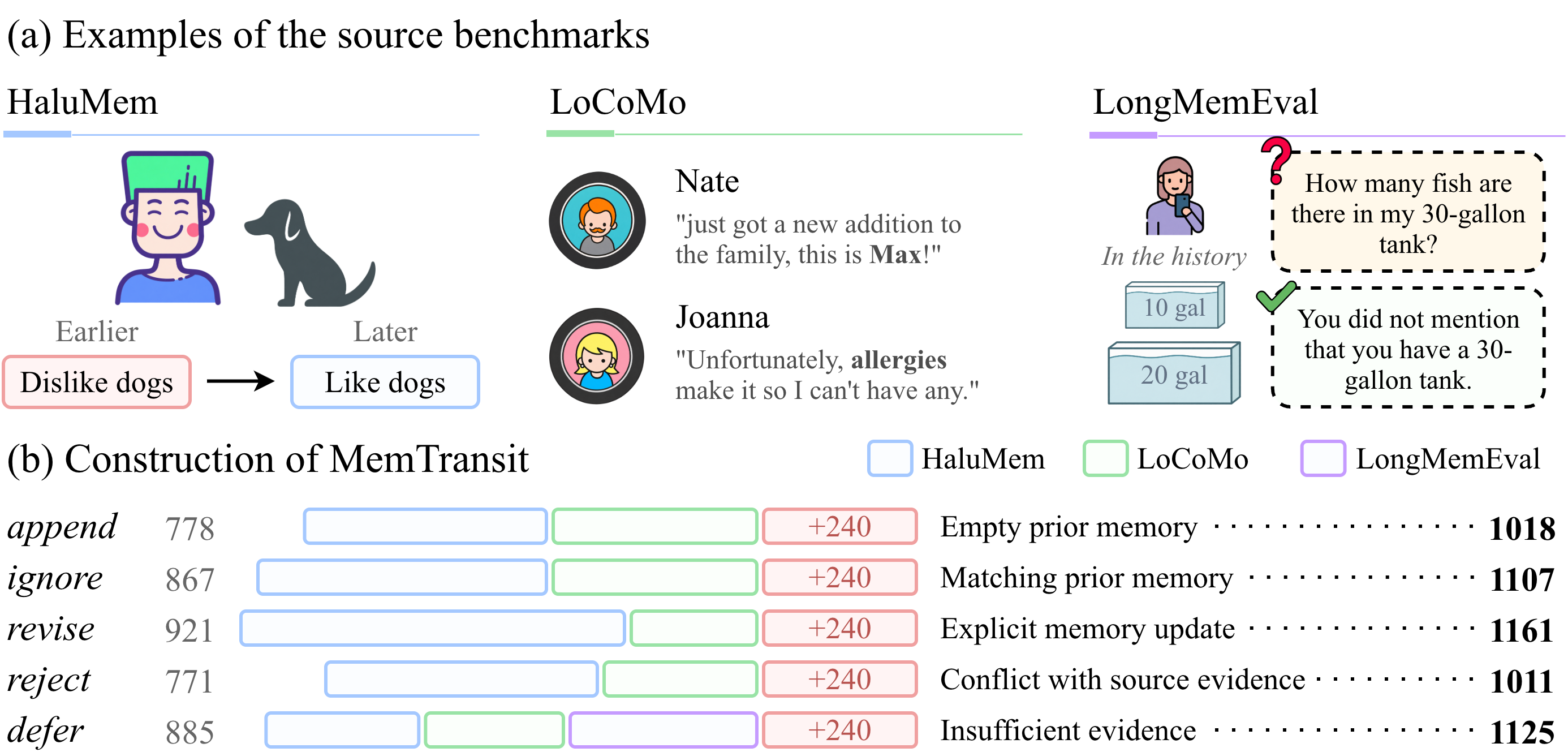}
\caption{MemTransit contains 5,422 examples from three benchmarks, including 240 per action added by selecting and annotating additional passages.}
\label{fig:bench}
\end{figure*}

\section{Related Work}

\paragraph{Memory Architectures and Retrieval.}
Existing research supports long-term interaction through experience,
reflection, and feedback
\cite{park2023generative,shinn2023reflexion,zhao2024expel,sumers2024cognitive}.
Hierarchical storage, associative structures, and latent representations
organize and retain information
\cite{packer2023memgpt,zhong2023memorybank,kang2025memoryos,xu2025amem,
wang2024memoryllm,wang2025mplus},
while retrieval mechanisms support its access and use
\cite{karpukhin2020dpr,khattab2020colbert,izacard2021fid,guu2020realm,
lewis2020rag,borgeaud2022retro,izacard2023atlas,gutierrez2024hipporag}.
We focus on how incoming information changes existing memory,
distinguishing retaining a claim from accepting it as fact.
This keeps unverified information available without treating it
as confirmed knowledge.

\paragraph{Learning Reliable Updates.}
Knowledge editing addresses factual revision
\cite{de2021editing,mitchell2022mend,mitchell2022serac,meng2022rome,
meng2023memit},
while temporal modeling and fact verification address temporal validity
and evidential support
\cite{trivedi2017knowevolve,dasgupta2018hyte,goel2020diachronic,
dhingra2022timeaware,thorne2018fever,schuster2021vitaminc}.
Memory policies are also optimized through task performance, process
feedback, and candidate-update comparisons
\cite{chhikara2025mem0,yan2025memoryr1,yu2026agemem,xu2026memcoe,
yang2026trustmem}.
MemTact executes alternative actions from the same initial memory and
input, comparing their results against annotated correct memory states.
This grounds action supervision in concrete memory changes, specifying
which contents should be preserved, modified, or discarded.

\paragraph{Evaluating Long-Term Memory.}
Existing benchmarks cover long-term question answering, knowledge
updates, forgetting, and hallucinations and omissions in memory
processing
\cite{bai2024longbench,zhang2024infinitebench,xu2022goldfish,
maharana2024locomo,wu2024longmemeval,tan2025membench,
hu2026memoryagentbench,chen2025halumem}.
MemTransit jointly annotates five update actions, target records, and
post-update memory contents, distinguishing accepted, pending, and
rejected information.
Evaluation thus examines action selection, target identification,
and content preservation, while tracking error propagation and
accumulation across successive updates.

\section{MemTransit Benchmark}
\label{sec:benchmark}

As new information arrives, models must decide how memory should change and make those changes correctly. MemTransit evaluates both: choosing revise is insufficient if the wrong record changes or outdated information remains. Each example includes an incoming statement, the current memory, and evidence then available. Reference annotations specify the correct action, affected records, and resulting memory.

\subsection{Memory and Evidence Determine Updates}
\label{subsec:data_construction}

We build examples from three benchmarks, each supporting different update decisions (Figure~\ref{fig:bench}). With LoCoMo~\cite{maharana2024locomo}, we keep a dialogue fact unchanged and vary what is already in memory. The fact requires append when memory is empty and ignore when it is already recorded. HaluMem~\cite{chen2025halumem} supplies explicit corrections to earlier facts, showing which record to revise and what should replace it. We use HaluMem claims that source evidence contradicts for reject, and claims that remain unresolved for defer. From LongMemEval~\cite{wu2024longmemeval}, we select questions that assume facts the conversation history does not establish. We use each assumed fact as the incoming statement and label it defer.

The dataset contains 1,018 append, 1,161 revise, 1,107 ignore, 1,011 reject, and 1,125 defer examples. Reference states distinguish accepted facts, pending claims, and historical/rejected content. Reject and defer can preserve the same accepted facts; only defer keeps the incoming claim pending.

\begin{figure*}[t]
    \centering
    \includegraphics[width=\textwidth]{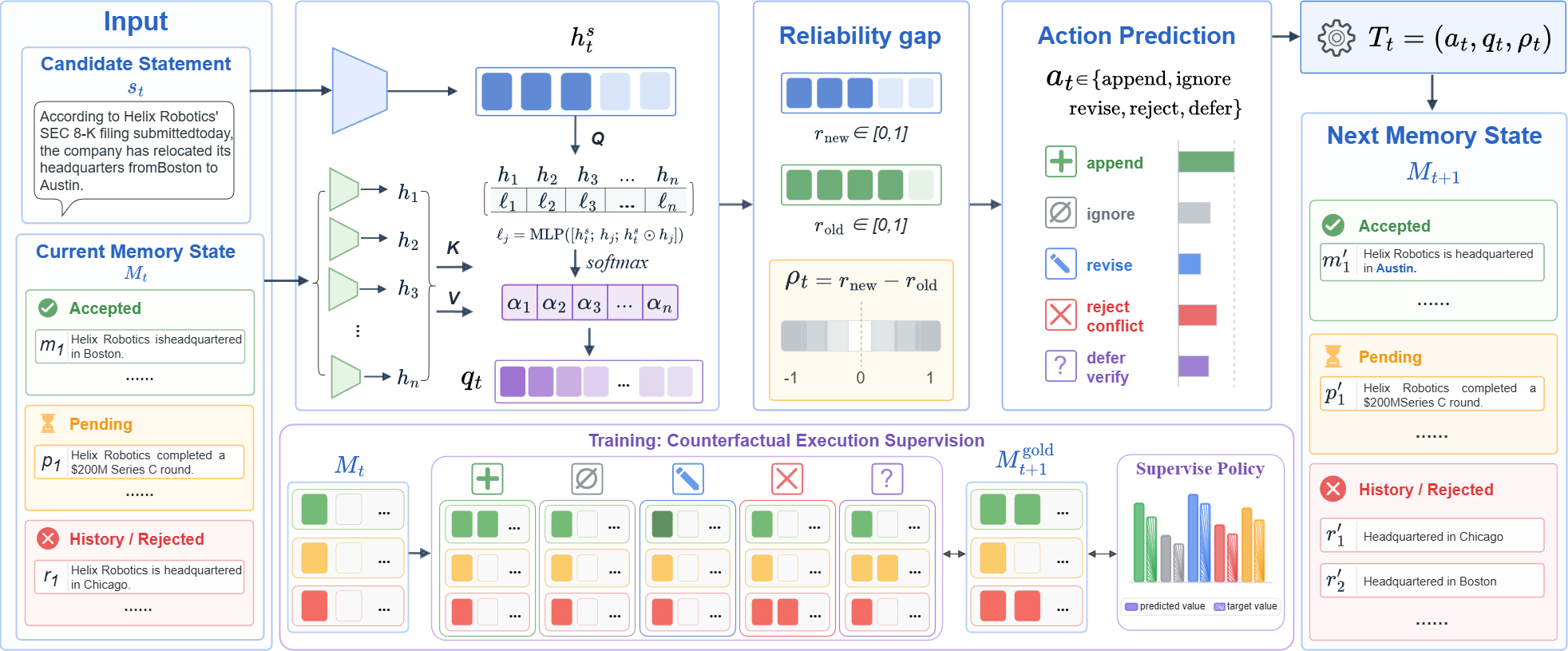}
    \caption{Overview of MemTact. The policy uses relevant records and relative evidence reliability to choose updates, and learns by comparing candidate results with annotated memory states.}
    \label{fig:method}
\end{figure*}

\subsection{Checking Actions and Updated Memory}
\label{subsec:annotation_quality}

Multiple annotators independently determine the correct action, which records it affects, and what content to retain, using existing records and evidence about timing and sources. We resolve disagreements and check reliability through independent reannotation. We apply each annotated action to its records and check that the resulting memory matches the annotated memory.

We keep source conversations and entities separate across splits, yielding 4,338 training, 542 validation, and 542 test examples.

\subsection{Following Errors Across Updates}
\label{subsec:benchmark_evaluation}

Single-step evaluation starts each update from its supplied memory. Sequential evaluation uses the model's updated memory as input for the next step. If the model misses a correction, it keeps the old value. Repeating that correction then requires revise in the model's memory but ignore in the reference memory. Comparing memory after each step and at the end shows where errors begin, how they affect later decisions, and whether they persist. Appendix~\ref{app:memtransit} details benchmark construction.

\section{Methodology}
\label{sec:methodology}

\subsection{Problem Formulation}
\label{subsec:problem_formulation}

Given an incoming statement $s_t$ and current memory $\mathcal M_t$, memory updating selects an action $a_t$ and, when required, a target record $\xi_t$ to produce $\mathcal M_{t+1}$.
The objective is to incorporate supported information, revise outdated records, and retain claims pending verification.
A correct update must therefore determine both which information to retain and the status under which it is retained.

As shown in Figure~\ref{fig:method}, MemTact first locates relevant records, compares the evidence for incoming and stored information, and predicts an update action.
A deterministic executor then updates the memory.
During training, different actions are executed independently from the same initial state, and their outcomes are compared with the annotated next state to supervise action selection.

\begin{table}[t]
    \centering
    \caption{Effect of each update action on the incoming claim and previously accepted records.}
    \label{tab:action_semantics}
    \small
    \renewcommand{\arraystretch}{1.12}
    \begin{tabular*}{\columnwidth}{@{\extracolsep{\fill}}lll@{}}
        \toprule
        Action & Incoming claim & Existing record \\
        \midrule
        append & Accepted  & Unchanged \\
        revise & Accepted  & Archived as history \\
        ignore & Discarded & Unchanged \\
        reject & Rejected  & Unchanged \\
        defer  & Pending   & Unchanged \\
        \bottomrule
    \end{tabular*}
\end{table}

\subsection{Memory States and Action Semantics}
\label{subsec:memory_semantics}

To distinguish retaining information from accepting it as fact, we represent memory as $\mathcal M_t=(A_t,P_t,H_t)$.
Here, $A_t$ stores currently confirmed facts, $P_t$ retains claims pending verification, and $H_t$ preserves superseded records and rejected claims with distinct statuses.
Superseded records remain historical evidence, whereas rejected claims record information that was not accepted.

Five actions specify the destination of the incoming claim and the treatment of the relevant accepted record, as summarized in Table~\ref{tab:action_semantics}.

These actions distinguish updates with similar immediate effects.
Both append and revise accept the incoming claim, but revise also archives the superseded record as history.
Likewise, ignore, reject, and defer all preserve existing accepted contents, but leave different evidence available afterward.

A deterministic executor applies these rules, preserving unaffected records and relevant supersession or conflict links.
Each action therefore produces a well-defined memory change whose outcome can be directly inspected.
Training and inference use the same execution semantics.

\subsection{Evidence-Based Action Prediction}
\label{subsec:action_prediction}

Action prediction requires locating the relevant record and determining which update the available evidence supports.
A shared encoder represents the incoming statement and candidate memory records.
A grounding network scores their compatibility and forms a weighted memory representation:
\begin{equation}
q_t=\sum_{i=1}^{K_t}\alpha_{t,i}h_{t,i}.
\label{eq:target_grounding}
\end{equation}
Here, $h_{t,i}$ represents the $i$-th candidate record, $K_t$ is the number of candidates, and $\alpha_{t,i}$ is its normalized compatibility score.
The weighted representation retains information from multiple plausible targets during prediction.
When execution requires a discrete target, the model selects the highest-scoring valid record; revision targets must belong to Accepted.
Since selecting the correct action and locating the correct record are distinct requirements, we supervise target grounding separately.

The model then combines the incoming statement, grounded memory, and their differences and interactions with visible temporal, source, confidence, and conflict cues.
A reliability head estimates the evidential support for incoming and stored information, denoted by $r_t,\bar r_t\in[0,1]$.
The policy receives both estimates and their difference $\rho_t=r_t-\bar r_t$.
The difference indicates which side is better supported, while the individual estimates distinguish strong evidence on both sides from weak evidence on both sides.
These signals jointly inform action selection alongside temporal and conflict cues.

Each action has a learned representation and a fixed execution code.
The latter specifies its destination ledger, target requirements, and treatment of existing records.
A shared scoring network combines action features with the evidence representation and reliability estimates to produce $p_\theta(a\mid s_t,\mathcal M_t)$, allowing action scores to reflect both evidence relations and the corresponding memory changes.

\subsection{Counterfactual Execution Supervision}
\label{subsec:counterfactual_supervision}

To connect action selection with its consequences, we execute each $a\in\mathcal A$ independently from the same initial memory and input:
\begin{equation}
\widetilde{\mathcal M}_{t+1}^{a}
=
F(\mathcal M_t,e_t,a,\xi_t^{a}).
\label{eq:counterfactual_execution}
\end{equation}
Here, $F$ is the deterministic executor, $e_t$ is the incoming record, and $\xi_t^{a}$ is the execution target.
Each branch independently produces a possible next memory, allowing the effects of different actions to be compared under identical initial conditions.
Branches requiring a target use its annotation when available and the predicted target otherwise, reducing the influence of grounding errors on action comparisons.

We then compare each branch outcome with the annotated next state:
\begin{equation}
Q_t(a)
=
\sum_{\ell=1}^{3}
\omega_\ell
S_\ell
\bigl(
\widetilde{\mathcal M}_{t+1}^{a,\ell},
\mathcal M_{t+1}^{*,\ell}
\bigr).
\label{eq:counterfactual_quality}
\end{equation}
Here, $\ell$ identifies a ledger, and $S_\ell$ measures agreement between the resulting and annotated contents.
Comparing all three ledgers captures missing records, incorrect overwrites, and information assigned to the wrong state.
Even actions that leave Accepted unchanged can receive different scores depending on whether they retain required evidence in Pending or History/Rejected.
Accepted receives the highest weight because its contents directly affect subsequent memory access.

Execution scores are then converted into soft action targets:
\begin{equation}
\pi_t(a)
=
\frac{e^{Q_t(a)/\tau}}
{\sum_{b\in\mathcal A}e^{Q_t(b)/\tau}}.
\label{eq:counterfactual_target}
\end{equation}
Actions producing memories closer to the annotated state receive higher target probabilities.
Temperature $\tau$ controls the concentration of the distribution: smaller values emphasize the highest-scoring outcome, while larger values retain more information about the relative quality of different actions.
The target distribution thus reflects differences in execution quality across the five actions.

\begin{table*}[t]
\centering
\caption{Memory update quality on MemTransit. All values are reported as percentages.}
\label{tab:overall_performance}
\small
\setlength{\tabcolsep}{5pt}
\renewcommand{\arraystretch}{1.12}
\begin{tabular*}{\textwidth}{@{\extracolsep{\fill}}lccccc@{}}
\toprule
\textbf{Method}
& \shortstack{\textbf{Five-action}\\\textbf{Macro-F1} $\uparrow$}
& \shortstack{\textbf{Next-State}\\\textbf{Acc.} $\uparrow$}
& \shortstack{\textbf{Memory}\\\textbf{Pollution} $\downarrow$}
& \shortstack{\textbf{Conflict}\\\textbf{Preservation} $\uparrow$}
& \shortstack{\textbf{Write}\\\textbf{Rate}} \\
\midrule
MemGPT \citep{packer2023memgpt}
& 76.99 & 63.47 & 32.22 & 50.50 & 44.10 \\
Mem0 \citep{chhikara2025mem0}
& 68.89 & 54.80 & 34.33 & 51.49 & 49.45 \\
Memory-R1 \citep{yan2025memoryr1}
& 78.35 & 62.92 & 29.05 & 53.47 & 44.46 \\
AgeMem \citep{yu2026agemem}
& 73.86 & 60.70 & 28.38 & 49.50 & 42.25 \\
MemCoE \citep{xu2026memcoe}
& 74.73 & 61.07 & 28.39 & 50.50 & 43.54 \\
\midrule
Ours
& \textbf{82.71}
& \textbf{66.42}
& \textbf{25.34}
& \textbf{54.46}
& \textbf{40.77} \\
\bottomrule
\end{tabular*}
\end{table*}

We train the policy to match this distribution using the following loss:
\begin{equation}
\mathcal L_{\Delta}
=
-\sum_{a\in\mathcal A}
\pi_t(a)\log p_\theta(a\mid s_t,\mathcal M_t).
\label{eq:counterfactual_loss}
\end{equation}
This loss translates differences in retained, modified, and discarded information into supervision for action probabilities.
An action is encouraged when its predicted probability falls below its execution-derived target probability, while an overestimated action is downweighted.
Execution and state comparison construct the targets, and gradients update the policy through the cross-entropy loss.

\subsection{Joint Training and Inference}
\label{subsec:training_inference}

We jointly supervise action selection, target grounding, reliability estimation, and execution outcomes with the objective
\begin{equation}
\mathcal L
=
\mathcal L_a
+\lambda_\xi\mathcal L_\xi
+\lambda_r\mathcal L_r
+\lambda_\Delta\mathcal L_\Delta.
\label{eq:training_objective}
\end{equation}
Here, $\mathcal L_a$ is the cross-entropy loss on the annotated action, $\mathcal L_\xi$ supervises target selection, and $\mathcal L_r$ supervises reliability estimates for incoming and stored information. The coefficients weight the auxiliary losses. Action labels specify the correct operation, while the counterfactual targets in $\mathcal L_\Delta$ capture the effects of alternative actions. Grounding and reliability supervision help the model locate the relevant record and act on the available evidence.

At inference, MemTact uses the current memory, incoming statement, and visible metadata to select the highest-scoring action and, if required, its target. It then executes the action once. Appendix~\ref{app:hyperparameters} details hyperparameters.

\section{Experiment}

\subsection{Experimental Setup}
\label{subsec:experimental_setup}

\paragraph{Data and protocols.}
We evaluate MemTact on MemTransit using the splits described above.
Single-step evaluation measures action selection and next-state correctness from the supplied initial memory.
Sequential evaluation propagates predicted memory across steps to assess error accumulation.

We adapt MemGPT \cite{packer2023memgpt}, Mem0 \cite{chhikara2025mem0}, Memory-R1 \cite{yan2025memoryr1}, AgeMem \cite{yu2026agemem}, and MemCoE \cite{xu2026memcoe} to a common five-action and target-record interface, preserving their core memory management and decision mechanisms.
All methods share the Llama-3.1-8B-Instruct backbone \cite{dubey2024llama}, data splits, training annotations, and test evidence, with matched context and retrieval budgets and equal validation-tuning opportunities.
Each uses training information through its own learning or prompting mechanism.
At test time, a shared deterministic executor applies each method's predicted actions to its predicted target records.

\paragraph{Metrics.}
We report five-action macro-F1 and next-state accuracy, the latter requiring exact matches to annotated contents and states across all three resulting ledgers.
Memory pollution rate is the fraction of new or revised records written to Accepted that should not be accepted:
\begin{equation}
P=
\frac{\sum_{i=1}^{N}|\Delta A_i\setminus A_i^*|}
{\sum_{i=1}^{N}|\Delta A_i|},
\label{eq:memory_pollution}
\end{equation}
where $N$ counts evaluation examples, $\Delta A_i$ contains new or revised records written to Accepted, and $A_i^*$ is the annotated next Accepted ledger.
Matching follows next-state evaluation.
We also report the fraction of updates writing to Accepted.

Conflict evidence preservation accuracy is
\begin{equation}
C=\frac{1}{|\mathcal R|}\sum_{i\in\mathcal R}c_i,
\label{eq:conflict_preservation}
\end{equation}
where $\mathcal R$ contains reject-annotated examples, and $c_i=1$ if annotation-required existing facts remain in Accepted and the rejected claim is retained with rejected status in History/Rejected; otherwise, $c_i=0$.
Lower $P$ and higher macro-F1, state accuracy, and $C$ are better.

Sequential evaluation reports stepwise and final-state accuracy---exact state matches across steps and at sequence end, respectively---and performance by update step to assess error accumulation.

\begin{figure}[t]
    \centering
    \includegraphics[width=\linewidth]{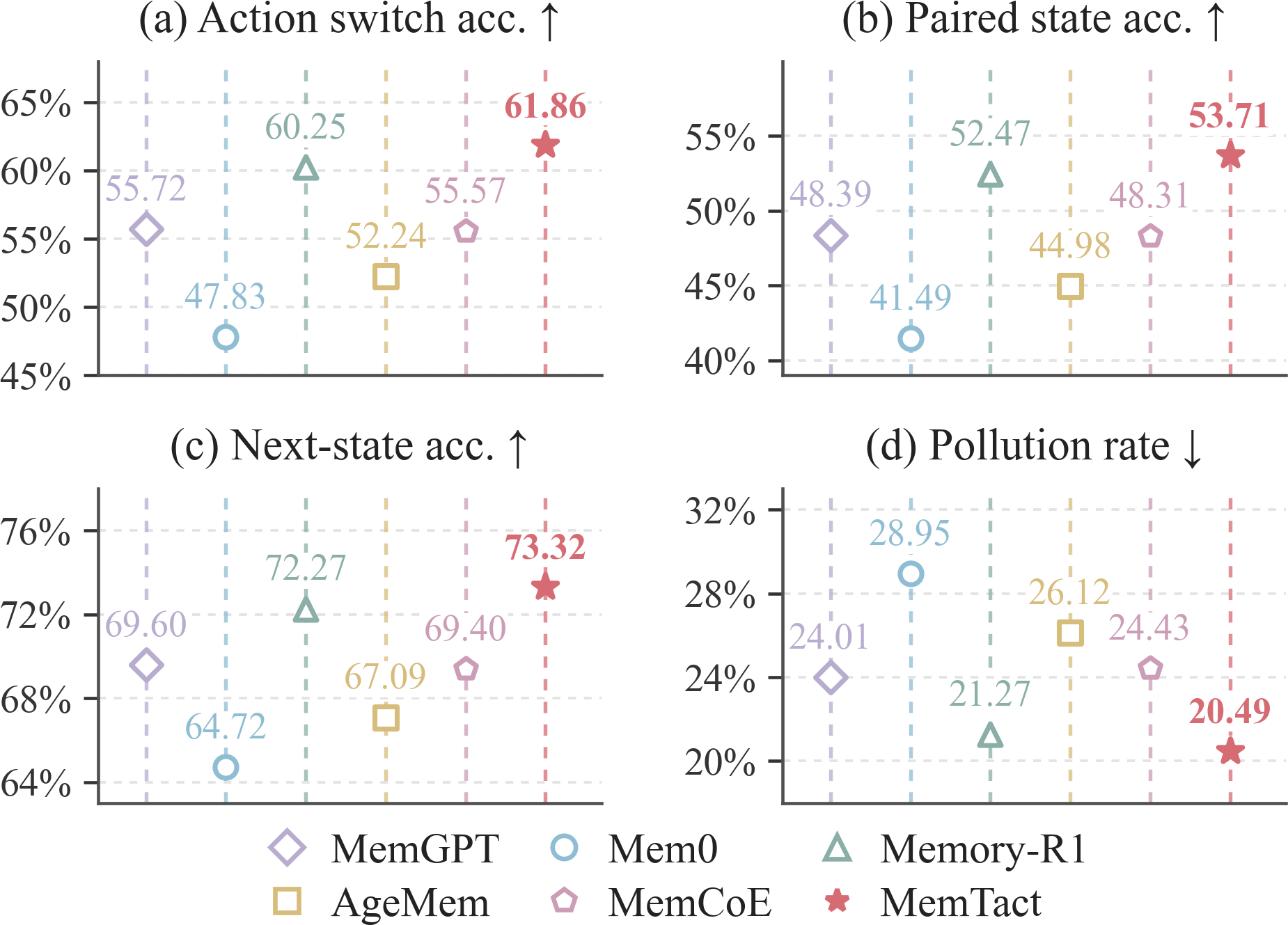}
    \caption{Memory update performance on paired cases where different evidence calls for different actions.}
    \label{fig:evidence}
\end{figure}

\begin{figure}[t]
    \centering
    \includegraphics[width=\linewidth]{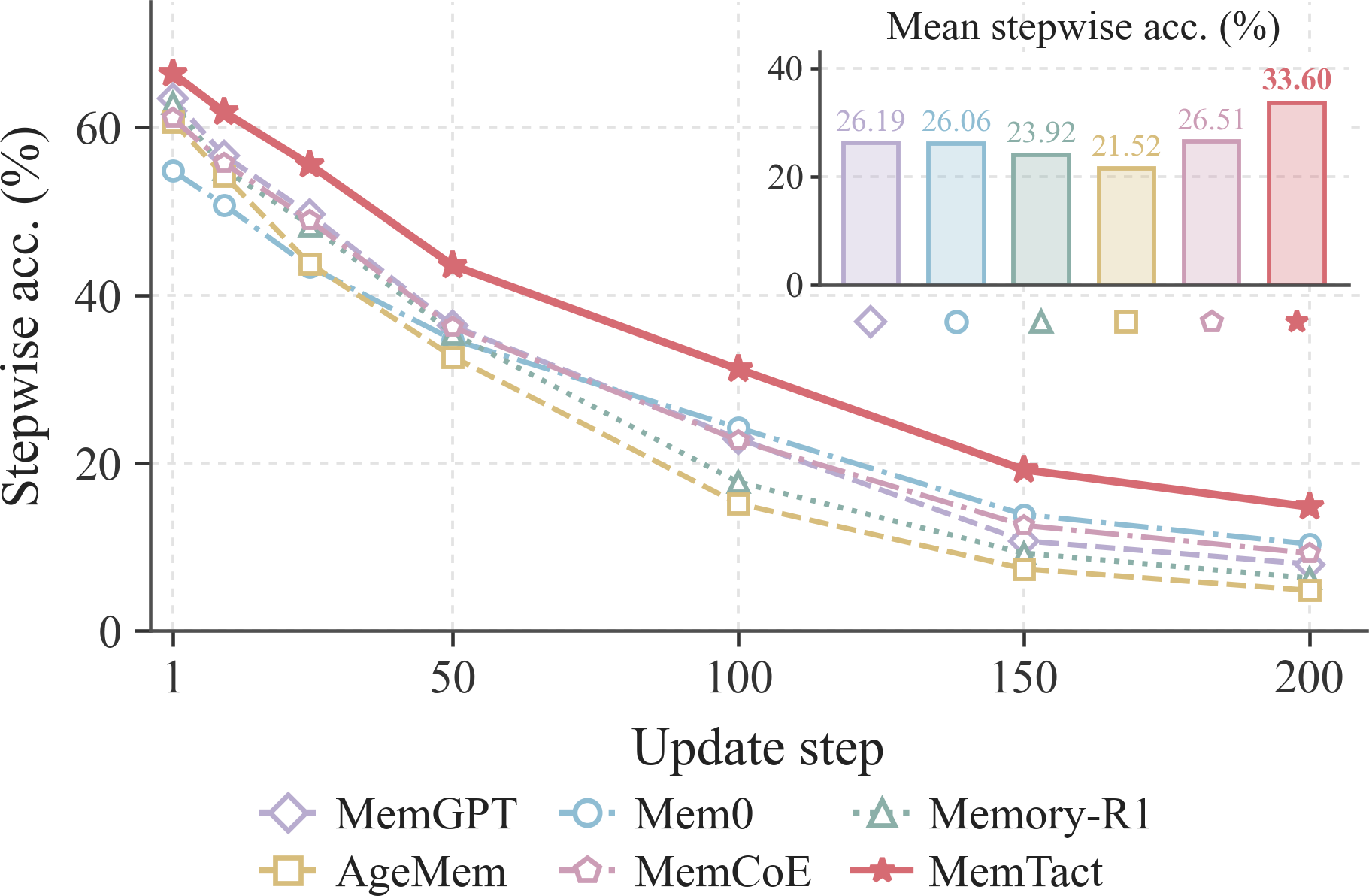}
    \caption{Stepwise memory-state accuracy over closed-loop rollouts; inset shows mean accuracy.}
    \label{fig:rollout}
\end{figure}

\begin{figure}[t]
    \centering
    \includegraphics[width=\linewidth]{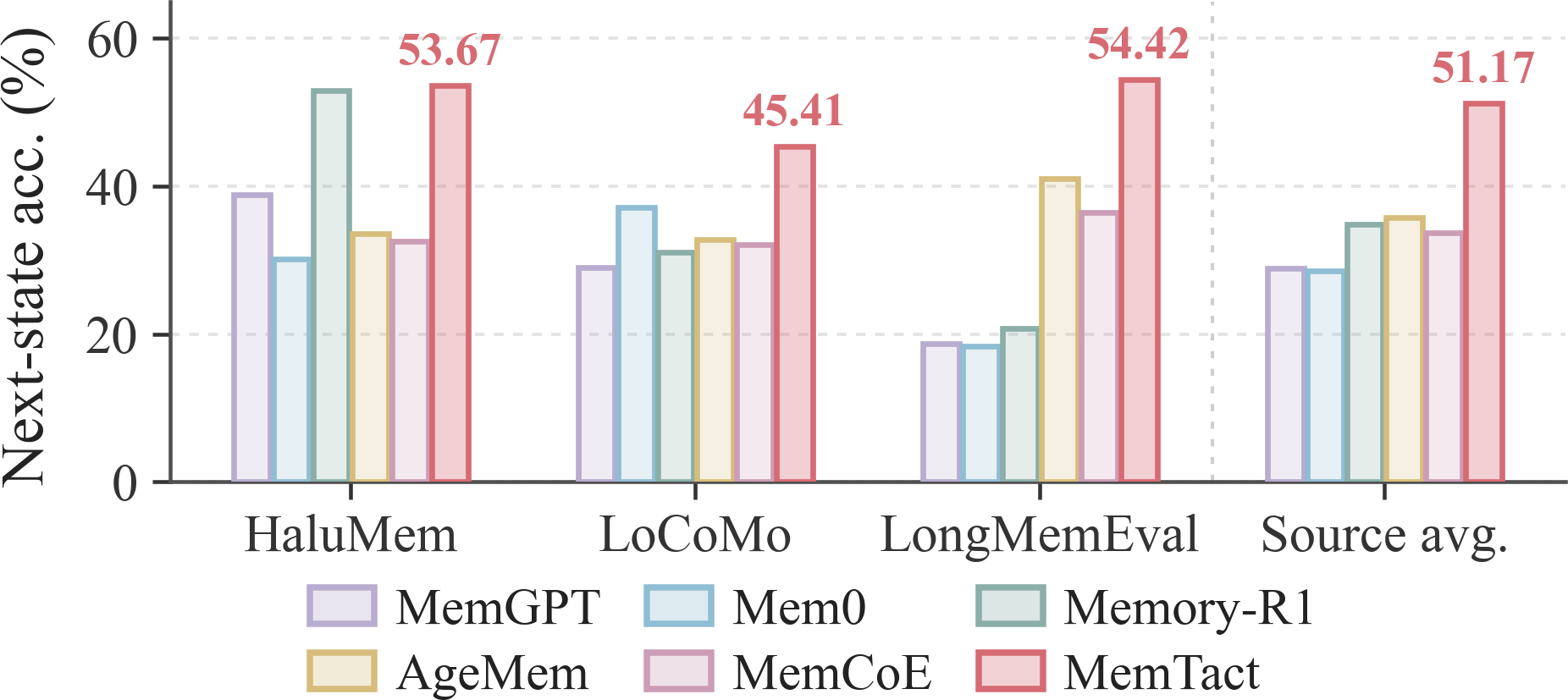}
    \caption{Next-state accuracy on unseen sources and their equally weighted average under cross-source evaluation.}
    \label{fig:cross}
\end{figure}

\begin{figure}[t]
    \centering
    \includegraphics[width=\linewidth]{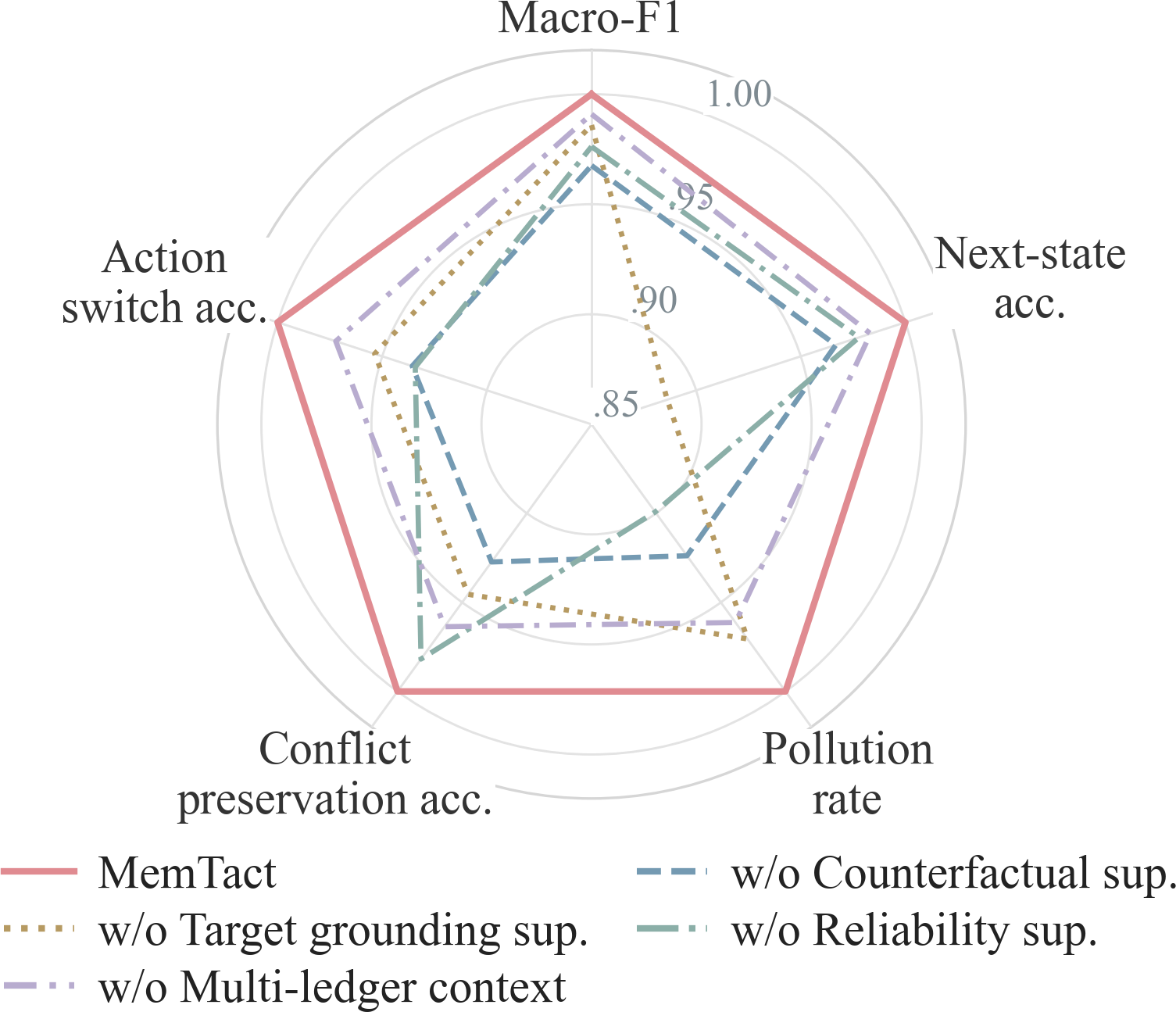}
    \caption{Ablation results normalized to full MemTact.}
    \label{fig:ablation}
\end{figure}

\subsection{Memory Update Quality}

On MemTransit, MemTact selects update actions more accurately and more often correctly updates record contents and statuses across all three ledgers.
All methods write to Accepted more frequently than the annotated rate (40.22\%), but MemTact comes closest and has the lowest fraction of incorrectly accepted records among those it writes to Accepted.
For incoming claims that should be rejected, MemTact more often preserves existing facts, marks the incoming claim as rejected, and stores it in the history ledger.
These results indicate that MemTact better distinguishes current facts from claims to retain as evidence of conflict.

\subsection{Changes in Evidence Shape Memory Updates}

We test whether MemTact distinguishes similar cases where the evidence requires different memory updates.
We pair cases requiring different actions and evaluate each pair jointly.
\emph{Action switch accuracy} requires correct actions in both cases; \emph{paired state accuracy} requires all three updated ledgers to match their annotations in both cases.

As Figure~\ref{fig:evidence} shows, MemTact outperforms all baselines on both paired metrics, more often switching actions correctly as evidence changes and producing correct memory states in both cases.
It also achieves the highest next-state accuracy and lowest memory pollution rate, indicating that its advantage extends to record acceptance.
Appendix~\ref{app:evidence_updates} further shows that MemTact preserves more conflicts at comparable write rates.

\subsection{Memory Reliability under Successive Updates}
\label{subsec:rollout}

To test whether single-step gains persist over longer sequences, we evaluate 200-step closed-loop rollouts, feeding each predicted memory into the next step.
This setting tests whether models maintain correct memory states while repeatedly updating their own outputs.

As Figure~\ref{fig:rollout} shows, MemTact outperforms all baselines at every reported checkpoint and achieves the highest mean stepwise and final-state accuracy.
Although all methods degrade as updates accumulate, MemTact consistently maintains a higher exact state-match rate.
These results show that its single-step advantage persists across successive updates, where each decision depends on previously maintained memory.
Appendix~\ref{app:successive_updates} further shows that MemTact maintains the most accurate states and lowest pollution across rollouts.

\begin{table}[t]
\centering
\caption{Downstream QA performance using updated memories. All values are percentages.}
\label{tab:downstream_utility}
\footnotesize
\setlength{\tabcolsep}{3pt}
\renewcommand{\arraystretch}{1.12}
\resizebox{\columnwidth}{!}{%
\begin{tabular}{@{}lccc@{}}
\toprule
\textbf{Method}
& \shortstack{\textbf{Overall}\\\textbf{QA F1} $\uparrow$}
& \shortstack{\textbf{Update-Focused}\\\textbf{QA F1} $\uparrow$}
& \shortstack{\textbf{Unsupported}\\\textbf{Answer Rate} $\downarrow$} \\
\midrule
MemGPT \citep{packer2023memgpt}
& 75.14 & 27.61 & 3.22 \\
Mem0 \citep{chhikara2025mem0}
& 78.43 & \textbf{41.89} & 2.89 \\
Memory-R1 \citep{yan2025memoryr1}
& 75.79 & 29.47 & 2.21 \\
AgeMem \citep{yu2026agemem}
& 77.31 & 39.18 & 3.03 \\
MemCoE \citep{xu2026memcoe}
& 75.60 & 30.74 & 2.40 \\
\midrule
Ours
& \textbf{78.76} & 39.25 & \textbf{2.18} \\
\bottomrule
\end{tabular}%
}
\end{table}

\subsection{Downstream Utility of Updated Memory}
\label{subsec:downstream_utility}

We evaluate how updated memories support question answering along three dimensions: overall QA F1 measures answer quality across all questions, update-focused QA F1 focuses on questions that require updated information, and unsupported answer rate measures the proportion of answers lacking supporting evidence. As shown in Table~\ref{tab:downstream_utility}, MemTact achieves the highest overall QA F1 and competitive performance on both update-focused QA and unsupported answer rate.

These results show that memories maintained by MemTact effectively support downstream question answering, with utility reflected in overall answer quality, the use of updated information, and evidential support for answers. Complementing the preceding memory-state evaluations, the downstream results further demonstrate MemTact's practical value in question-answering tasks.

\subsection{Cross-Source Generalization}
\label{subsec:cross_source}

To evaluate generalization to unseen data sources, we train on two sources and test on the third, using HaluMem, LoCoMo, and LongMemEval as the test source in turn. All methods follow the same protocol. We report next-state accuracy for each test source and the equally weighted average across sources.

As shown in Figure~\ref{fig:cross}, MemTact achieves the highest next-state accuracy on all three test sources and the highest source average. While the relative performance of the baselines varies across sources, MemTact consistently leads, with larger margins over the strongest baseline on LoCoMo and LongMemEval. These results demonstrate that its memory-update capabilities transfer to sources unseen during training, yielding more reliable memory states under cross-source evaluation. Appendix~\ref{app:cross_source_results} shows that our method achieves the highest next-state accuracy on every held-out source.

\subsection{Component and Supervision Ablations}
\label{subsec:ablation}

We separately remove counterfactual execution supervision, target grounding supervision, reliability supervision, and multi-ledger context. As shown in Figure~\ref{fig:ablation}, full MemTact performs best across all evaluated dimensions. Removing counterfactual execution supervision degrades performance throughout, indicating that learning how different actions change memory, alongside learning the correct action labels, helps the model select appropriate actions and produce correct memory states.

The components make distinct contributions: removing target grounding supervision causes the largest drop in next-state accuracy, removing reliability supervision yields the largest increase in memory pollution, and removing multi-ledger context consistently degrades performance. These findings support the value of accurately locating update targets, assessing evidence reliability, and jointly using memory records with different statuses. Appendix~\ref{app:component_ablations} shows that counterfactual supervision contributes most to conflict preservation.

\section{Conclusion}

Reliable persistent memory depends on what agents retain and how they update it, as update errors can persist and affect later decisions. We introduced MemTransit, which unifies five update actions and jointly annotates actions, target records, and resulting states across three ledgers. By linking decisions to their consequences, MemTransit supports evaluating both action selection and resulting memory states. We further introduced MemTact, which learns by comparing the outcomes of executing alternative actions. Experiments show that MemTact improves action selection and next-state accuracy, reduces memory pollution, and better preserves conflict evidence. These findings demonstrate that alternative updates' outcomes provide an effective learning signal for more reliable memory updates.

\section*{Limitations}

Our evaluation covers three data sources, with all methods using a shared action interface and executor. Future work can further examine MemTact's applicability to additional data sources and native memory systems.


\bibliography{custom}

@inproceedings{park2023generative,
  title = {Generative Agents: Interactive Simulacra of Human Behavior},
  author = {Park, Joon Sung and O'Brien, Joseph C. and
            Cai, Carrie J. and Morris, Meredith Ringel and
            Liang, Percy and Bernstein, Michael S.},
  booktitle = {Proceedings of the 36th Annual ACM Symposium on User Interface Software and Technology},
  year = {2023},
  publisher = {Association for Computing Machinery},
  doi = {10.1145/3586183.3606763},
  url = {https://doi.org/10.1145/3586183.3606763}
}

@inproceedings{shinn2023reflexion,
  title = {{Reflexion}: Language Agents with Verbal Reinforcement Learning},
  author = {Shinn, Noah and Cassano, Federico and Gopinath, Ashwin and
            Narasimhan, Karthik and Yao, Shunyu},
  booktitle = {Advances in Neural Information Processing Systems},
  volume = {36},
  year = {2023},
  doi = {10.52202/075280-0377},
  url = {https://proceedings.neurips.cc/paper/2023/hash/1b44b878bb782e6954cd888628510e90-Abstract-Conference.html}
}

@inproceedings{zhao2024expel,
  title = {{ExpeL}: {LLM} Agents Are Experiential Learners},
  author = {Zhao, Andrew and Huang, Daniel and Xu, Quentin and
            Lin, Matthieu and Liu, Yong-Jin and Huang, Gao},
  booktitle = {Proceedings of the AAAI Conference on Artificial Intelligence},
  volume = {38},
  number = {17},
  pages = {19632--19642},
  year = {2024},
  doi = {10.1609/aaai.v38i17.29936},
  url = {https://ojs.aaai.org/index.php/AAAI/article/view/29936}
}

@article{sumers2024cognitive,
  title = {Cognitive Architectures for Language Agents},
  author = {Sumers, Theodore R. and Yao, Shunyu and
            Narasimhan, Karthik and Griffiths, Thomas L.},
  journal = {Transactions on Machine Learning Research},
  year = {2024},
  url = {https://openreview.net/pdf?id=1i6ZCvflQJ}
}

@article{packer2023memgpt,
  title = {{MemGPT}: Towards {LLMs} as Operating Systems},
  author = {Packer, Charles and Wooders, Sarah and Lin, Kevin and
            Fang, Vivian and Patil, Shishir G. and Stoica, Ion and
            Gonzalez, Joseph E.},
  journal = {arXiv preprint arXiv:2310.08560},
  year = {2023},
  url = {https://arxiv.org/abs/2310.08560}
}

@inproceedings{zhong2023memorybank,
  title = {{MemoryBank}: Enhancing Large Language Models with Long-Term Memory},
  author = {Zhong, Wanjun and Guo, Lianghong and Gao, Qiqi and
            Ye, He and Wang, Yanlin},
  booktitle = {Proceedings of the AAAI Conference on Artificial Intelligence},
  volume = {38},
  number = {17},
  pages = {19724--19731},
  year = {2024},
  doi = {10.1609/aaai.v38i17.29946},
  url = {https://ojs.aaai.org/index.php/AAAI/article/view/29946}
}

@inproceedings{kang2025memoryos,
  title = {Memory {OS} of {AI} Agent},
  author = {Kang, Jiazheng and Ji, Mingming and Zhao, Zhe and Bai, Ting},
  booktitle = {Proceedings of the 2025 Conference on Empirical Methods in Natural Language Processing},
  year = {2025},
  month = nov,
  address = {Suzhou, China},
  publisher = {Association for Computational Linguistics},
  pages = {25961--25970},
  doi = {10.18653/v1/2025.emnlp-main.1318},
  url = {https://aclanthology.org/2025.emnlp-main.1318/}
}

@inproceedings{xu2025amem,
  title = {{A-Mem}: Agentic Memory for {LLM} Agents},
  author = {Xu, Wujiang and Liang, Zujie and Mei, Kai and
            Gao, Hang and Tan, Juntao and Zhang, Yongfeng},
  booktitle = {Advances in Neural Information Processing Systems},
  volume = {38},
  year = {2025},
  doi = {10.52202/085713-0593},
  url = {https://proceedings.nips.cc/paper_files/paper/2025/hash/19909c36f51abc4856b4560aff3d36d6-Abstract-Conference.html}
}

@inproceedings{wang2024memoryllm,
  title = {{MEMORYLLM}: Towards Self-Updatable Large Language Models},
  author = {Wang, Yu and Gao, Yifan and Chen, Xiusi and Jiang, Haoming and
            Li, Shiyang and Yang, Jingfeng and Yin, Qingyu and Li, Zheng and
            Li, Xian and Yin, Bing and Shang, Jingbo and Mcauley, Julian},
  booktitle = {Proceedings of the 41st International Conference on Machine Learning},
  series = {Proceedings of Machine Learning Research},
  volume = {235},
  pages = {50453--50466},
  publisher = {PMLR},
  year = {2024},
  url = {https://proceedings.mlr.press/v235/wang24s.html}
}

@inproceedings{wang2025mplus,
  title = {{M+}: Extending {MemoryLLM} with Scalable Long-Term Memory},
  author = {Wang, Yu and Krotov, Dmitry and Hu, Yuanzhe and Gao, Yifan and
            Zhou, Wangchunshu and Mcauley, Julian and Gutfreund, Dan and
            Feris, Rogerio and He, Zexue},
  booktitle = {Proceedings of the 42nd International Conference on Machine Learning},
  series = {Proceedings of Machine Learning Research},
  volume = {267},
  pages = {63308--63323},
  publisher = {PMLR},
  year = {2025},
  url = {https://proceedings.mlr.press/v267/wang25au.html}
}

@inproceedings{karpukhin2020dpr,
  title = {Dense Passage Retrieval for Open-Domain Question Answering},
  author = {Karpukhin, Vladimir and Oguz, Barlas and Min, Sewon and
            Lewis, Patrick and Wu, Ledell and Edunov, Sergey and
            Chen, Danqi and Yih, Wen-tau},
  booktitle = {Proceedings of the 2020 Conference on Empirical Methods in Natural Language Processing ({EMNLP})},
  year = {2020},
  month = nov,
  address = {Online},
  publisher = {Association for Computational Linguistics},
  pages = {6769--6781},
  doi = {10.18653/v1/2020.emnlp-main.550},
  url = {https://aclanthology.org/2020.emnlp-main.550/}
}

@inproceedings{khattab2020colbert,
  title = {{ColBERT}: Efficient and Effective Passage Search via Contextualized Late Interaction over {BERT}},
  author = {Khattab, Omar and Zaharia, Matei},
  booktitle = {Proceedings of the 43rd International ACM SIGIR Conference on Research and Development in Information Retrieval},
  year = {2020},
  publisher = {Association for Computing Machinery},
  pages = {39--48},
  doi = {10.1145/3397271.3401075},
  url = {https://doi.org/10.1145/3397271.3401075}
}

@inproceedings{izacard2021fid,
  title = {Leveraging Passage Retrieval with Generative Models for Open Domain Question Answering},
  author = {Izacard, Gautier and Grave, Edouard},
  booktitle = {Proceedings of the 16th Conference of the European Chapter of the Association for Computational Linguistics: Main Volume},
  year = {2021},
  month = apr,
  address = {Online},
  publisher = {Association for Computational Linguistics},
  pages = {874--880},
  doi = {10.18653/v1/2021.eacl-main.74},
  url = {https://aclanthology.org/2021.eacl-main.74/}
}

@inproceedings{guu2020realm,
  title = {Retrieval Augmented Language Model Pre-Training},
  author = {Guu, Kelvin and Lee, Kenton and Tung, Zora and
            Pasupat, Panupong and Chang, Mingwei},
  booktitle = {Proceedings of the 37th International Conference on Machine Learning},
  series = {Proceedings of Machine Learning Research},
  volume = {119},
  pages = {3929--3938},
  publisher = {PMLR},
  year = {2020},
  url = {https://proceedings.mlr.press/v119/guu20a.html}
}

@inproceedings{lewis2020rag,
  title = {Retrieval-Augmented Generation for Knowledge-Intensive {NLP} Tasks},
  author = {Lewis, Patrick and Perez, Ethan and Piktus, Aleksandra and
            Petroni, Fabio and Karpukhin, Vladimir and Goyal, Naman and
            K{\"u}ttler, Heinrich and Lewis, Mike and Yih, Wen-tau and
            Rockt{\"a}schel, Tim and Riedel, Sebastian and Kiela, Douwe},
  booktitle = {Advances in Neural Information Processing Systems},
  volume = {33},
  year = {2020},
  url = {https://proceedings.nips.cc/paper/2020/hash/6b493230205f780e1bc26945df7481e5-Abstract.html}
}

@inproceedings{borgeaud2022retro,
  title = {Improving Language Models by Retrieving from Trillions of Tokens},
  author = {Borgeaud, Sebastian and Mensch, Arthur and Hoffmann, Jordan and
            Cai, Trevor and Rutherford, Eliza and Millican, Katie and
            Van Den Driessche, George Bm and Lespiau, Jean-Baptiste and
            Damoc, Bogdan and Clark, Aidan and De Las Casas, Diego and
            Guy, Aurelia and Menick, Jacob and Ring, Roman and
            Hennigan, Tom and Huang, Saffron and Maggiore, Loren and
            Jones, Chris and Cassirer, Albin and Brock, Andy and
            Paganini, Michela and Irving, Geoffrey and Vinyals, Oriol and
            Osindero, Simon and Simonyan, Karen and Rae, Jack and
            Elsen, Erich and Sifre, Laurent},
  booktitle = {Proceedings of the 39th International Conference on Machine Learning},
  series = {Proceedings of Machine Learning Research},
  volume = {162},
  pages = {2206--2240},
  publisher = {PMLR},
  year = {2022},
  url = {https://proceedings.mlr.press/v162/borgeaud22a.html}
}

@article{izacard2023atlas,
  title = {{Atlas}: Few-shot Learning with Retrieval Augmented Language Models},
  author = {Izacard, Gautier and Lewis, Patrick and Lomeli, Maria and
            Hosseini, Lucas and Petroni, Fabio and Schick, Timo and
            Dwivedi-Yu, Jane and Joulin, Armand and Riedel, Sebastian and
            Grave, Edouard},
  journal = {Journal of Machine Learning Research},
  volume = {24},
  number = {251},
  pages = {1--43},
  year = {2023},
  url = {https://www.jmlr.org/papers/v24/23-0037.html}
}

@inproceedings{gutierrez2024hipporag,
  title = {{HippoRAG}: Neurobiologically Inspired Long-Term Memory for Large Language Models},
  author = {Jim{\'e}nez Guti{\'e}rrez, Bernal and Shu, Yiheng and
            Gu, Yu and Yasunaga, Michihiro and Su, Yu},
  booktitle = {Advances in Neural Information Processing Systems},
  volume = {37},
  year = {2024},
  doi = {10.52202/079017-1902},
  url = {https://proceedings.neurips.cc/paper_files/paper/2024/hash/6ddc001d07ca4f319af96a3024f6dbd1-Abstract.html}
}

@inproceedings{chhikara2025mem0,
  title = {{Mem0}: Building Production-Ready {AI} Agents with Scalable Long-Term Memory},
  author = {Chhikara, Prateek and Khant, Dev and Aryan, Saket and
            Singh, Taranjeet and Yadav, Deshraj},
  booktitle = {ECAI 2025},
  series = {Frontiers in Artificial Intelligence and Applications},
  volume = {413},
  pages = {2993--3000},
  publisher = {IOS Press},
  year = {2025},
  doi = {10.3233/FAIA251160},
  url = {https://doi.org/10.3233/FAIA251160}
}

@inproceedings{yan2025memoryr1,
  title = {{Memory-R1}: Enhancing Large Language Model Agents to Manage and Utilize Memories via Reinforcement Learning},
  author = {Yan, Sikuan and Yang, Xiufeng and Huang, Zuchao and
            Nie, Ercong and Ding, Zifeng and Li, Zonggen and
            Ma, Xiaowen and Bi, Jinhe and Kersting, Kristian and
            Pan, Jeff Z. and Schuetze, Hinrich and Tresp, Volker and
            Ma, Yunpu},
  booktitle = {Proceedings of the 64th Annual Meeting of the Association for Computational Linguistics (Volume 1: Long Papers)},
  year = {2026},
  month = jul,
  address = {San Diego, California, United States},
  publisher = {Association for Computational Linguistics},
  pages = {12805--12825},
  doi = {10.18653/v1/2026.acl-long.583},
  url = {https://aclanthology.org/2026.acl-long.583/}
}

@inproceedings{yu2026agemem,
  title = {Agentic Memory: Learning Unified Long-Term and Short-Term Memory Management for Large Language Model Agents},
  author = {Yu, Yi and Yao, Liuyi and Xie, Yuexiang and Tan, Qingquan and
            Feng, Jiaqi and Li, Yaliang and Wu, Libing},
  booktitle = {Proceedings of the 64th Annual Meeting of the Association for Computational Linguistics (Volume 1: Long Papers)},
  year = {2026},
  month = jul,
  address = {San Diego, California, United States},
  publisher = {Association for Computational Linguistics},
  pages = {21457--21483},
  doi = {10.18653/v1/2026.acl-long.981},
  url = {https://aclanthology.org/2026.acl-long.981/}
}

@inproceedings{xu2026memcoe,
  title = {Learning How and What to Memorize: Cognition-Inspired Two-Stage Optimization for Evolving Memory},
  author = {Xu, Derong and Liu, Shuochen and Luo, Pengfei and
            Jia, Pengyue and Zhang, Yingyi and Wen, Yi and Deng, Yimin and
            Zhang, Wenlin and Chen, Enhong and Zhao, Xiangyu and Xu, Tong},
  booktitle = {Proceedings of the 64th Annual Meeting of the Association for Computational Linguistics (Volume 1: Long Papers)},
  year = {2026},
  month = jul,
  address = {San Diego, California, United States},
  publisher = {Association for Computational Linguistics},
  pages = {44987--45011},
  doi = {10.18653/v1/2026.acl-long.2084},
  url = {https://aclanthology.org/2026.acl-long.2084/}
}

@inproceedings{yang2026trustmem,
  title = {{TrustMem}: Learning Trustworthy Memory Consolidation for {LLM} Agents with Long-Term Memory},
  author = {Yang, Tianyu and Paul, Sudipta and Srinivasan, Vijay and
            Kulkarni, Vivek and Chappidi, Srinivas},
  booktitle = {Findings of the Association for Computational Linguistics: EMNLP 2026},
  year = {2026},
  month = oct,
  address = {Budapest, Hungary},
  publisher = {Association for Computational Linguistics},
  pages = {51948--51965},
  doi = {10.18653/2026.findings-emnlp.2419},
  url = {https://aclanthology.org/2026.findings-emnlp.2419/}
}

@inproceedings{de2021editing,
  title = {Editing Factual Knowledge in Language Models},
  author = {De Cao, Nicola and Aziz, Wilker and Titov, Ivan},
  booktitle = {Proceedings of the 2021 Conference on Empirical Methods in Natural Language Processing},
  year = {2021},
  month = nov,
  address = {Online and Punta Cana, Dominican Republic},
  publisher = {Association for Computational Linguistics},
  pages = {6491--6506},
  doi = {10.18653/v1/2021.emnlp-main.522},
  url = {https://aclanthology.org/2021.emnlp-main.522/}
}

@inproceedings{mitchell2022mend,
  title = {Fast Model Editing at Scale},
  author = {Mitchell, Eric and Lin, Charles and Bosselut, Antoine and
            Finn, Chelsea and Manning, Christopher D.},
  booktitle = {International Conference on Learning Representations},
  year = {2022},
  url = {https://openreview.net/forum?id=0DcZxeWfOPt}
}

@inproceedings{mitchell2022serac,
  title = {Memory-Based Model Editing at Scale},
  author = {Mitchell, Eric and Lin, Charles and Bosselut, Antoine and
            Manning, Christopher D. and Finn, Chelsea},
  booktitle = {Proceedings of the 39th International Conference on Machine Learning},
  series = {Proceedings of Machine Learning Research},
  volume = {162},
  pages = {15817--15831},
  publisher = {PMLR},
  year = {2022},
  url = {https://proceedings.mlr.press/v162/mitchell22a.html}
}

@inproceedings{meng2022rome,
  title = {Locating and Editing Factual Associations in {GPT}},
  author = {Meng, Kevin and Bau, David and Andonian, Alex and Belinkov, Yonatan},
  booktitle = {Advances in Neural Information Processing Systems},
  volume = {35},
  year = {2022},
  doi = {10.52202/068431-1262},
  url = {https://proceedings.neurips.cc/paper_files/paper/2022/hash/6f1d43d5a82a37e89b0665b33bf3a182-Abstract-Conference.html}
}

@inproceedings{meng2023memit,
  title = {Mass-Editing Memory in a Transformer},
  author = {Meng, Kevin and Sharma, Arnab Sen and Andonian, Alex and
            Belinkov, Yonatan and Bau, David},
  booktitle = {International Conference on Learning Representations},
  year = {2023},
  url = {https://openreview.net/forum?id=MkbcAHIYgyS}
}

@inproceedings{trivedi2017knowevolve,
  title = {{Know-Evolve}: Deep Temporal Reasoning for Dynamic Knowledge Graphs},
  author = {Trivedi, Rakshit and Dai, Hanjun and Wang, Yichen and Song, Le},
  booktitle = {Proceedings of the 34th International Conference on Machine Learning},
  series = {Proceedings of Machine Learning Research},
  volume = {70},
  pages = {3462--3471},
  publisher = {PMLR},
  year = {2017},
  url = {https://proceedings.mlr.press/v70/trivedi17a.html}
}

@inproceedings{dasgupta2018hyte,
  title = {{HyTE}: Hyperplane-based Temporally aware Knowledge Graph Embedding},
  author = {Dasgupta, Shib Sankar and Ray, Swayambhu Nath and Talukdar, Partha},
  booktitle = {Proceedings of the 2018 Conference on Empirical Methods in Natural Language Processing},
  year = {2018},
  month = oct # "-" # nov,
  address = {Brussels, Belgium},
  publisher = {Association for Computational Linguistics},
  pages = {2001--2011},
  doi = {10.18653/v1/D18-1225},
  url = {https://aclanthology.org/D18-1225/}
}

@inproceedings{goel2020diachronic,
  title = {Diachronic Embedding for Temporal Knowledge Graph Completion},
  author = {Goel, Rishab and Kazemi, Seyed Mehran and
            Brubaker, Marcus and Poupart, Pascal},
  booktitle = {Proceedings of the AAAI Conference on Artificial Intelligence},
  volume = {34},
  number = {04},
  pages = {3988--3995},
  year = {2020},
  doi = {10.1609/aaai.v34i04.5815},
  url = {https://ojs.aaai.org/index.php/AAAI/article/view/5815}
}

@article{dhingra2022timeaware,
  title = {Time-Aware Language Models as Temporal Knowledge Bases},
  author = {Dhingra, Bhuwan and Cole, Jeremy R. and
            Eisenschlos, Julian Martin and Gillick, Daniel and
            Eisenstein, Jacob and Cohen, William W.},
  journal = {Transactions of the Association for Computational Linguistics},
  volume = {10},
  pages = {257--273},
  year = {2022},
  publisher = {MIT Press},
  doi = {10.1162/tacl_a_00459},
  url = {https://aclanthology.org/2022.tacl-1.15/}
}

@inproceedings{thorne2018fever,
  title = {{FEVER}: a Large-scale Dataset for Fact Extraction and {VER}ification},
  author = {Thorne, James and Vlachos, Andreas and
            Christodoulopoulos, Christos and Mittal, Arpit},
  booktitle = {Proceedings of the 2018 Conference of the North American Chapter of the Association for Computational Linguistics: Human Language Technologies, Volume 1 (Long Papers)},
  year = {2018},
  month = jun,
  address = {New Orleans, Louisiana},
  publisher = {Association for Computational Linguistics},
  pages = {809--819},
  doi = {10.18653/v1/N18-1074},
  url = {https://aclanthology.org/N18-1074/}
}

@inproceedings{schuster2021vitaminc,
  title = {Get Your Vitamin {C}! Robust Fact Verification with Contrastive Evidence},
  author = {Schuster, Tal and Fisch, Adam and Barzilay, Regina},
  booktitle = {Proceedings of the 2021 Conference of the North American Chapter of the Association for Computational Linguistics: Human Language Technologies},
  year = {2021},
  month = jun,
  address = {Online},
  publisher = {Association for Computational Linguistics},
  pages = {624--643},
  doi = {10.18653/v1/2021.naacl-main.52},
  url = {https://aclanthology.org/2021.naacl-main.52/}
}

@inproceedings{bai2024longbench,
  title = {{LongBench}: A Bilingual, Multitask Benchmark for Long Context Understanding},
  author = {Bai, Yushi and Lv, Xin and Zhang, Jiajie and Lyu, Hongchang and
            Tang, Jiankai and Huang, Zhidian and Du, Zhengxiao and
            Liu, Xiao and Zeng, Aohan and Hou, Lei and Dong, Yuxiao and
            Tang, Jie and Li, Juanzi},
  booktitle = {Proceedings of the 62nd Annual Meeting of the Association for Computational Linguistics (Volume 1: Long Papers)},
  year = {2024},
  month = aug,
  address = {Bangkok, Thailand},
  publisher = {Association for Computational Linguistics},
  pages = {3119--3137},
  doi = {10.18653/v1/2024.acl-long.172},
  url = {https://aclanthology.org/2024.acl-long.172/}
}

@inproceedings{zhang2024infinitebench,
  title = {$\infty${Bench}: Extending Long Context Evaluation Beyond 100{K} Tokens},
  author = {Zhang, Xinrong and Chen, Yingfa and Hu, Shengding and
            Xu, Zihang and Chen, Junhao and Hao, Moo Khai and
            Han, Xu and Thai, Zhen Leng and Wang, Shuo and
            Liu, Zhiyuan and Sun, Maosong},
  booktitle = {Proceedings of the 62nd Annual Meeting of the Association for Computational Linguistics (Volume 1: Long Papers)},
  year = {2024},
  month = aug,
  address = {Bangkok, Thailand},
  publisher = {Association for Computational Linguistics},
  pages = {15262--15277},
  doi = {10.18653/v1/2024.acl-long.814},
  url = {https://aclanthology.org/2024.acl-long.814/}
}

@inproceedings{xu2022goldfish,
  title = {Beyond Goldfish Memory: Long-Term Open-Domain Conversation},
  author = {Xu, Jing and Szlam, Arthur and Weston, Jason},
  booktitle = {Proceedings of the 60th Annual Meeting of the Association for Computational Linguistics (Volume 1: Long Papers)},
  year = {2022},
  month = may,
  address = {Dublin, Ireland},
  publisher = {Association for Computational Linguistics},
  pages = {5180--5197},
  doi = {10.18653/v1/2022.acl-long.356},
  url = {https://aclanthology.org/2022.acl-long.356/}
}

@inproceedings{maharana2024locomo,
  title = {Evaluating Very Long-Term Conversational Memory of {LLM} Agents},
  author = {Maharana, Adyasha and Lee, Dong-Ho and Tulyakov, Sergey and
            Bansal, Mohit and Barbieri, Francesco and Fang, Yuwei},
  booktitle = {Proceedings of the 62nd Annual Meeting of the Association for Computational Linguistics (Volume 1: Long Papers)},
  year = {2024},
  month = aug,
  address = {Bangkok, Thailand},
  publisher = {Association for Computational Linguistics},
  pages = {13851--13870},
  doi = {10.18653/v1/2024.acl-long.747},
  url = {https://aclanthology.org/2024.acl-long.747/}
}

@inproceedings{wu2024longmemeval,
  title = {{LongMemEval}: Benchmarking Chat Assistants on Long-Term Interactive Memory},
  author = {Wu, Di and Wang, Hongwei and Yu, Wenhao and Zhang, Yuwei and
            Chang, Kai-Wei and Yu, Dong},
  booktitle = {The Thirteenth International Conference on Learning Representations},
  year = {2025},
  url = {https://proceedings.iclr.cc/paper_files/paper/2025/hash/d813d324dbf0598bbdc9c8e79740ed01-Abstract-Conference.html}
}

@inproceedings{tan2025membench,
  title = {{MemBench}: Towards More Comprehensive Evaluation on the Memory of {LLM}-based Agents},
  author = {Tan, Haoran and Zhang, Zeyu and Ma, Chen and Chen, Xu and
            Dai, Quanyu and Dong, Zhenhua},
  booktitle = {Findings of the Association for Computational Linguistics: ACL 2025},
  year = {2025},
  month = jul,
  address = {Vienna, Austria},
  publisher = {Association for Computational Linguistics},
  pages = {19336--19352},
  doi = {10.18653/v1/2025.findings-acl.989},
  url = {https://aclanthology.org/2025.findings-acl.989/}
}

@inproceedings{hu2026memoryagentbench,
  title = {Evaluating Memory in {LLM} Agents via Incremental Multi-Turn Interactions},
  author = {Hu, Yuanzhe and Wang, Yu and McAuley, Julian},
  booktitle = {The Fourteenth International Conference on Learning Representations},
  year = {2026},
  url = {https://openreview.net/forum?id=DT7JyQC3MR}
}

@inproceedings{chen2025halumem,
  title = {{HaluMem}: Evaluating Hallucinations in Memory Systems of Agents},
  author = {Chen, Ding and Niu, Simin and Li, Kehang and Liu, Peng and
            Zheng, Xiangping and Tang, Bo and Li, Xinchi and
            Xiong, Feiyu and Li, Zhiyu},
  booktitle = {Proceedings of the 2026 Conference on Empirical Methods in Natural Language Processing},
  year = {2026},
  month = oct,
  address = {Budapest, Hungary},
  publisher = {Association for Computational Linguistics},
  pages = {5073--5108},
  doi = {10.18653/2026.emnlp-main.220},
  url = {https://aclanthology.org/2026.emnlp-main.220/}
}

@article{dubey2024llama,
  title   = {The {Llama 3} Herd of Models},
  author  = {Dubey, Abhimanyu and Jauhri, Abhinav and Pandey, Abhinav and others},
  journal = {arXiv preprint arXiv:2407.21783},
  year    = {2024},
  url     = {https://arxiv.org/abs/2407.21783v1}
}

\appendix

\section{Theoretical Analysis}
\label{app:theoretical_analysis}

We analyze how counterfactual execution distinguishes memory updates, shapes action supervision, and relates to execution quality.

For each action $a$, the execution target $\xi_t^a$ specifies the record affected by the operation when a target is required.
All branches are evaluated against the same annotated next state $\mathcal M_{t+1}^{*}$, as defined in Eq.~\eqref{eq:counterfactual_quality}.
The analysis conditions on the resulting branch scores, which are treated as fixed supervision when optimizing the action distribution.

We assume that each ledger agreement function $S_\ell$ takes values in $[0,1]$, attains one when all evaluated fields match, and is strictly lower when an evaluated field differs.
Record status is included in the History/Rejected comparison.
All ledger weights are positive, and their sum is denoted by $W=\sum_{\ell=1}^{3}\omega_\ell$.
All logarithms are natural.

\subsection{Ledger-Level Separation}

The three-ledger representation distinguishes memory updates through their consequences across Accepted, Pending, and History/Rejected.

\paragraph{Proposition 1 (Ledger-level separation).}
Suppose the executed outcome of action $a$ matches the annotated next state on every evaluated ledger field.
Any action $b$ whose outcome differs on at least one evaluated field satisfies $Q_t(a)>Q_t(b)$.

\noindent\textit{Proof.}
Since action $a$ matches all evaluated fields, every ledger agreement score attains its maximum, giving a total score of $W$.
For action $b$, at least one positively weighted ledger receives strictly lower agreement, while the remaining scores cannot exceed one.
Consequently, $Q_t(b)$ is strictly smaller.

This property distinguishes actions that preserve identical Accepted contents but produce different outcomes elsewhere.
Under Table~\ref{tab:action_semantics}, ignore discards the incoming claim, reject retains it with rejected status, and defer retains it in Pending.
Their outcomes are distinguishable whenever these operations produce different evaluated ledger contents.

Branches agreeing on all evaluated fields receive identical counterfactual scores and target probabilities.
The annotated action loss $\mathcal L_a$ provides action-level supervision for such outcome-equivalent branches.

\subsection{Counterfactual Supervision as Entropic Optimization}

The target distribution in Eq.~\eqref{eq:counterfactual_target} admits a variational interpretation.

Let $\Delta(\mathcal A)$ denote the probability simplex over the five actions.
The Shannon entropy of a distribution $p$ is
$\mathcal H(p)=-\sum_{a\in\mathcal A}p(a)\log p(a)$.

\paragraph{Proposition 2 (Entropy-regularized execution quality).}
For $\tau>0$, the counterfactual distribution $\pi_t$ is given by

\begin{equation}
\pi_t = \arg\max_{p\in\Delta(\mathcal A)}
\left[\sum_{a\in\mathcal A}p(a)Q_t(a)+\tau\mathcal H(p)\right].
\label{eq:entropy_optimality}
\end{equation}

\noindent\textit{Proof.}
The probability simplex is compact and convex, and the objective is strictly concave.
As any action probability approaches zero, the corresponding entropy derivative diverges to positive infinity, while the derivative of the execution-quality term remains finite.
Thus, the maximizer lies in the interior of the simplex.

Introducing a Lagrange multiplier for the normalization constraint, the stationarity conditions yield probabilities proportional to the exponentials of execution scores scaled by $\tau$.
Normalization recovers Eq.~\eqref{eq:counterfactual_target}.
Strict concavity guarantees that this stationary point is the unique global maximizer.

The variational formulation characterizes the balance between expected execution quality and action entropy.
Higher-scoring branches receive greater probability, while $\tau$ controls the concentration of the target distribution.

Let
$D_t=D_{\mathrm{KL}}(\pi_t\Vert p_\theta(\cdot\mid s_t,\mathcal M_t))$,
where $D_{\mathrm{KL}}$ denotes KL divergence.
The counterfactual cross-entropy in Eq.~\eqref{eq:counterfactual_loss} satisfies

\begin{equation}
\mathcal L_\Delta
=
\mathcal H(\pi_t)+D_t.
\label{eq:counterfactual_kl}
\end{equation}

For fixed counterfactual targets, minimizing $\mathcal L_\Delta$ minimizes the divergence between the predicted policy and the execution-derived distribution.

\paragraph{Lemma 1 (Joint action supervision).}
Consider the action-dependent training objective
$\mathcal L_a+\lambda_\Delta\mathcal L_\Delta$,
where $\mathcal L_a$ is the negative log-probability of the annotated action $a_t$.
For fixed $\pi_t$ and $\lambda_\Delta\geq0$, its minimizer over nonparametric distributions in $\Delta(\mathcal A)$ is

\begin{equation}
p_t^{\mathrm{opt}}(a)
=
\frac{
\mathbf{1}_{\{a_t\}}(a)
+\lambda_\Delta\pi_t(a)
}{
1+\lambda_\Delta
},
\label{eq:joint_action_optimality}
\end{equation}

where $\mathbf{1}_{\{a_t\}}(a)$ is one for the annotated action and zero otherwise.

\noindent\textit{Proof.}
The objective is a weighted cross-entropy.
The coefficient associated with each negative log-probability is the sum of its annotated one-hot target and its counterfactual target weighted by $\lambda_\Delta$.
These nonnegative coefficients sum to $1+\lambda_\Delta$.
By Gibbs' inequality, cross-entropy over the probability simplex is minimized by the normalized coefficient distribution, yielding Eq.~\eqref{eq:joint_action_optimality}.

Let $a_t^Q$ denote any action maximizing $Q_t(a)$.
If the annotated action reproduces the annotated next state on all evaluated fields, Proposition 1 implies that $a_t$ is a counterfactual score maximizer.
The two supervision signals then favor a common score-maximizing action.
Moreover, $p_t^{\mathrm{opt}}$ uniquely favors $a_t$, even when multiple actions achieve the same maximum execution score.

When the annotated action has a lower counterfactual score, the two signals induce a tradeoff.
A score-maximizing action receives higher probability than the annotated action under $p_t^{\mathrm{opt}}$ precisely when
$\lambda_\Delta\bigl(\pi_t(a_t^Q)-\pi_t(a_t)\bigr)>1$.

For every action other than $a_t$, its probability under $p_t^{\mathrm{opt}}$ is proportional to its counterfactual target probability.
Therefore, every maximizer of $p_t^{\mathrm{opt}}$ belongs either to the annotated action or to the set of counterfactual score maximizers, with ties possible when their probabilities coincide.

Since $\tau>0$ and all execution scores are finite, $\pi_t$ assigns positive probability to every action.
For every finite $\lambda_\Delta$, the joint optimum therefore differs from $\pi_t$, and its KL divergence from $\pi_t$ is strictly positive.
As $\lambda_\Delta$ increases without bound, $p_t^{\mathrm{opt}}$ approaches $\pi_t$.

\subsection{Execution Quality and Policy Consistency}

We next establish how counterfactual target matching relates to expected execution quality.

Define the range of counterfactual scores as
$R_t=\max_{a\in\mathcal A}Q_t(a)-\min_{a\in\mathcal A}Q_t(a)$.

\paragraph{Proposition 3 (Expected execution-quality bound).}
For any predicted action distribution, the expected counterfactual quality gap satisfies

\begin{equation}
\begin{aligned}
&Q_t(a_t^Q)
-\sum_{a\in\mathcal A}
p_\theta(a\mid s_t,\mathcal M_t)Q_t(a)\\
&\qquad\leq
\tau\mathcal H(\pi_t)
+R_t\sqrt{\frac{D_t}{2}}.
\end{aligned}
\label{eq:execution_quality_bound}
\end{equation}

\noindent\textit{Proof.}
By Proposition 2, the entropy-regularized objective attained by $\pi_t$ is at least that of a distribution assigning all probability mass to $a_t^Q$.
Therefore, the gap between the maximum execution score and its expectation under $\pi_t$ is bounded by $\tau\mathcal H(\pi_t)$.

Because the execution scores have range $R_t$, the difference between their expectations under $\pi_t$ and the predicted distribution is bounded by $R_t$ times the total variation distance between these distributions.
Pinsker's inequality bounds this distance by $\sqrt{D_t/2}$.
Combining both bounds establishes Eq.~\eqref{eq:execution_quality_bound}.

The bound separates the entropy contribution of the counterfactual target from the divergence between that target and the predicted policy.
For fixed branch scores and temperature, the entropy term is constant, while the counterfactual loss directly penalizes the divergence term.
Since $\mathcal H(\pi_t)\leq\log 5$ and $R_t\leq W$, the bound also yields a uniform expression in terms of temperature and ledger weights.

Lemma 1 characterizes the nonparametric optimum of joint action supervision.
For the learned parameterized policy, the following result provides a sufficient condition for greedy action selection to agree with the highest-scoring counterfactual branch.

Let $\hat a_t$ denote the greedy action selected by $p_\theta$.

\paragraph{Corollary 1 (Greedy action consistency).}
Suppose $a_t^Q$ uniquely maximizes $Q_t(a)$.
Define the counterfactual probability margin as
$\gamma_t=\pi_t(a_t^Q)-\max_{a\ne a_t^Q}\pi_t(a)$.
Then

\begin{equation}
D_t<\frac{\gamma_t^2}{2}
\quad\Longrightarrow\quad
\hat a_t=a_t^Q.
\label{eq:greedy_consistency}
\end{equation}

\noindent\textit{Proof.}
The softmax transformation preserves the ordering of execution scores, so $a_t^Q$ uniquely maximizes $\pi_t$.

Write $p(a)$ for the predicted action probability at the current step.
Pinsker's inequality bounds the $\ell_1$ distance between $p$ and $\pi_t$ by $\sqrt{2D_t}$, which is strictly smaller than $\gamma_t$ under the stated condition.

For any competing action $b$, the predicted probability difference satisfies
$p(a_t^Q)-p(b)\geq\gamma_t-\|p-\pi_t\|_1>0$.
Thus, $a_t^Q$ remains the unique greedy action.

Together, these results characterize the theoretical properties of counterfactual memory updating.
Ledger-level comparison separates distinguishable update consequences.
Joint action supervision combines annotated operations with execution-derived preferences.
Counterfactual target matching provides a bound on single-step expected execution quality and a sufficient condition for consistent greedy action selection.

\section{MemTransit Construction and Annotation}
\label{app:memtransit}

\subsection{Candidate Extraction}
\label{app:memtransit_construction}

The dataset combines 4,222 previously labeled examples with 1,200 additions, balanced at 240 per action. The additions include 240 LoCoMo session observations, each producing an \emph{append}/\emph{ignore} pair through the memory variation described in Section~\ref{sec:benchmark}. These 480 examples carry high-confidence silver labels and are confined to training.

For HaluMem revisions, we retrieve the first original record linked to an annotated update and require its text to differ from the replacement. For rejection candidates, we select conflict questions with cited evidence and use the first cited record as the target. Deferral candidates are drawn from answerability-boundary questions, questions without cited evidence, or reference answers identifying information as unknown or unreported. Annotators determine the final action from the source evidence, distinguishing a supported correction, a contradicted claim, and an unresolved claim.

\subsection{Annotation and Quality Control}
\label{app:memtransit_annotation}
\label{app:memtransit_quality}

The annotation team comprised 18 university students and two university professors. Annotators independently recorded the action, target position when required, and successor memory. Position-based targets distinguish records with similar text. Disagreements were adjudicated by first judging the source evidence independently, then examining the annotations and their rationales. For conflicts, adjudication checked entity identity and temporal scope before comparing source support.

After the references were frozen, independent reviewers reannotated a stratified sample without access to the original judgments or adjudication records. The audit compared actions and successor states, together with target positions for revisions and rejections. Reference actions were also replayed from their supplied memories to check candidate placement, the designated record's disposition, and preservation of unrelated records. Independent review assessed the evidence supporting each update; replay checked consistency between the annotated action, target, and resulting state.

\begin{table*}[t]
\centering
\caption{Closed-loop memory quality over 200 updates. All values are percentages.}
\label{tab:rollout_details}
\small
\setlength{\tabcolsep}{3pt}
\renewcommand{\arraystretch}{1.12}
\begin{tabular*}{\textwidth}{@{\extracolsep{\fill}}lrrrrrrrccc@{}}
\toprule
& \multicolumn{7}{c}{\textbf{Exact State Accuracy by Update Step} $\uparrow$}
& \multicolumn{3}{c}{\textbf{Full Rollout}} \\
\cmidrule(lr){2-8}\cmidrule(l){9-11}
\textbf{Method}
& \textbf{1} & \textbf{10} & \textbf{25} & \textbf{50}
& \textbf{100} & \textbf{150} & \textbf{200}
& \shortstack{\textbf{Mean Stepwise}\\\textbf{Acc.} $\uparrow$}
& \shortstack{\textbf{Memory}\\\textbf{Pollution} $\downarrow$}
& \shortstack{\textbf{Write}\\\textbf{Rate}} \\
\midrule
MemGPT
& 63.47 & 56.64 & 49.63 & 36.35 & 22.88 & 10.70 & 7.93
& 26.19 & 34.82 & 42.23 \\
Mem0
& 54.80 & 50.74 & 43.36 & 34.69 & 24.17 & 13.84 & 10.33
& 26.06 & 37.82 & 48.50 \\
Memory-R1
& 62.92 & 54.98 & 48.34 & 35.42 & 17.71 & 9.23 & 6.27
& 23.92 & 34.15 & 44.32 \\
AgeMem
& 60.70 & 54.24 & 43.73 & 32.66 & 15.13 & 7.38 & 4.80
& 21.52 & 32.81 & 43.59 \\
MemCoE
& 61.07 & 55.72 & 48.89 & 36.16 & 22.69 & 12.55 & 9.23
& 26.51 & 31.59 & 44.33 \\
\midrule
MemTact
& \textbf{66.42} & \textbf{61.81} & \textbf{55.54} & \textbf{43.54}
& \textbf{31.18} & \textbf{19.19} & \textbf{14.76}
& \textbf{33.60} & \textbf{28.00} & 38.71 \\
\bottomrule
\end{tabular*}
\end{table*}

\subsection{Split Construction}
\label{app:memtransit_splits}

Examples sharing a group key, entity key, parent source item, source conversation, or canonical incoming statement were linked before splitting. We assigned entire connected components in decreasing order of size, keeping both directly and indirectly related examples together. The final split sizes are reported in Section~\ref{sec:benchmark}.

Cross-source evaluation uses the 4,222 previously labeled examples from HaluMem-hard, LoCoMo-derived, and LongMemEval-derived data. Each fold holds out one source for testing and partitions the other two into training and validation groups. A deterministic hash selects validation groups with a target proportion of $10\%$. The source average weights the three held-out sources equally.

\section{Additional Experimental Results}
\label{app:additional_results}

\subsection{Write Frequency and Conflict Preservation}
\label{app:evidence_updates}

Write frequency leaves substantial differences in conflict handling unresolved. On paired updates, MemTact and Memory-R1 write at nearly identical rates, yet MemTact more often preserves the complete conflict record (Table~\ref{tab:paired_conflict}). Similar commitment rates can therefore accompany different retention of disputed evidence.

Conflict preservation requires retaining both the supported Accepted fact and the rejected claim with its correct status. An \emph{ignore} decision can preserve the accepted fact while discarding the incoming challenge; \emph{reject} additionally records what was disputed. The action switches in Figure~\ref{fig:evidence} thus affect what later updates can recover from memory, even when the fact available for retrieval remains unchanged. MemTact's advantage concerns this richer outcome: preserving the accepted conclusion together with the evidence of the conflict.

\subsection{Memory Reliability over Successive Updates}
\label{app:successive_updates}
\label{app:memtransit_evaluation}

The rollout results distinguish final-state performance from reliability throughout an interaction (Table~\ref{tab:rollout_details}). Mem0 starts with the lowest exact-state accuracy among the baselines and finishes with the highest, while producing the highest pollution and trailing other baselines in mean stepwise accuracy. Its final ranking therefore obscures differences in the memories supplied to earlier decisions.

MemTact leads at every reported checkpoint and in mean stepwise accuracy, while producing the lowest pollution among accepted writes. These gains cover both complete-state agreement and the correctness of accepted content across the rollout. The advantage extends to the intermediate memories that subsequent updates repeatedly reuse.

\begin{table}[t]
\centering
\caption{Paired-update diagnostics. All values are percentages.}
\label{tab:paired_conflict}
\small
\setlength{\tabcolsep}{3pt}
\renewcommand{\arraystretch}{1.12}
\begin{tabular*}{\columnwidth}{@{\extracolsep{\fill}}lcc@{}}
\toprule
\textbf{Method}
& \shortstack{\textbf{Write}\\\textbf{Rate}}
& \shortstack{\textbf{Conflict}\\\textbf{Preservation} $\uparrow$} \\
\midrule
MemGPT    & 35.32 & 74.80 \\
Mem0      & 35.91 & 69.03 \\
Memory-R1 & 35.25 & 77.50 \\
AgeMem    & 35.91 & 71.98 \\
MemCoE    & 35.49 & 74.78 \\
\midrule
MemTact   & 35.03 & \textbf{78.63} \\
\bottomrule
\end{tabular*}
\end{table}

\begin{table*}[t]
\centering
\caption{Memory update quality by held-out source. All values are percentages. Only next-state accuracy is reported for LongMemEval because all reference actions are \emph{defer}.}
\label{tab:cross_source_details}
\small
\setlength{\tabcolsep}{3pt}
\renewcommand{\arraystretch}{1.12}
\begin{tabular*}{\textwidth}{@{\extracolsep{\fill}}lcccccccc@{}}
\toprule
& \multicolumn{3}{c}{\textbf{HaluMem}}
& \multicolumn{3}{c}{\textbf{LoCoMo}}
& \textbf{LongMemEval}
& \textbf{Source Avg.} \\
\cmidrule(lr){2-4}\cmidrule(lr){5-7}
\cmidrule(lr){8-8}\cmidrule(l){9-9}
\textbf{Method}
& \textbf{Acc.} $\uparrow$ & \textbf{F1} $\uparrow$ & \textbf{Poll.} $\downarrow$
& \textbf{Acc.} $\uparrow$ & \textbf{F1} $\uparrow$ & \textbf{Poll.} $\downarrow$
& \textbf{Acc.} $\uparrow$
& \textbf{Acc.} $\uparrow$ \\
\midrule
MemGPT
& 38.81 & 45.61 & 45.61
& 29.00 & 20.76 & 47.78
& 18.69 & 28.83 \\
Mem0
& 30.18 & 26.49 & 23.10
& 37.14 & 35.00 & 46.58
& 18.31 & 28.54 \\
Memory-R1
& 52.90 & 55.50 & \textbf{20.31}
& 31.00 & 35.56 & 43.28
& 20.74 & 34.88 \\
AgeMem
& 33.57 & 38.34 & 27.65
& 32.83 & 19.94 & 47.46
& 40.97 & 35.79 \\
MemCoE
& 32.52 & 34.78 & 43.49
& 32.10 & 24.25 & 45.64
& 36.49 & 33.70 \\
\midrule
MemTact
& \textbf{53.67} & \textbf{60.11} & 20.62
& \textbf{45.41} & \textbf{40.15} & \textbf{42.80}
& \textbf{54.42} & \textbf{51.17} \\
\bottomrule
\end{tabular*}
\end{table*}

\begin{table*}[t]
\centering
\caption{Memory update quality under component ablations. All values are percentages.}
\label{tab:ablation_details}
\small
\setlength{\tabcolsep}{4pt}
\renewcommand{\arraystretch}{1.12}
\begin{tabular*}{\textwidth}{@{\extracolsep{\fill}}lcccc@{}}
\toprule
\textbf{Variant}
& \shortstack{\textbf{Five-action}\\\textbf{Macro-F1} $\uparrow$}
& \shortstack{\textbf{Next-State}\\\textbf{Acc.} $\uparrow$}
& \shortstack{\textbf{Memory}\\\textbf{Pollution} $\downarrow$}
& \shortstack{\textbf{Conflict}\\\textbf{Preservation} $\uparrow$} \\
\midrule
w/o Counterfactual supervision
& 80.03 & 64.21 & 27.43 & 50.50 \\
w/o Target grounding supervision
& 81.55 & 58.86 & 26.11 & 51.49 \\
w/o Reliability supervision
& 80.74 & 64.94 & 28.18 & 53.47 \\
w/o Multi-ledger context
& 81.96 & 65.31 & 26.36 & 52.48 \\
\midrule
Full
& \textbf{82.71} & \textbf{66.42} & \textbf{25.34} & \textbf{54.46} \\
\bottomrule
\end{tabular*}
\end{table*}

\subsection{Cross-Source Generalization}
\label{app:cross_source_results}

MemTact achieves the highest next-state accuracy on every held-out source, while the strongest baseline changes across sources (Table~\ref{tab:cross_source_details}). Its source-average advantage is therefore supported by gains in all three folds.

The sources expose different aspects of transfer. On HaluMem, Memory-R1 has slightly lower pollution, whereas MemTact more often recovers the complete state: avoiding incorrect accepted writes and completing the required update are distinct outcomes. On LoCoMo, MemTact leads in both action macro-F1 and next-state accuracy, extending its advantage across operation selection and execution. LongMemEval evaluates unresolved claims, with \emph{defer} as the reference action throughout. Here, exact-state accuracy tests whether the incoming claim remains Pending while the existing memory is preserved. MemTact's strongest source-wise lead in state accuracy occurs on this evidence-limited subset.

\subsection{Component and Supervision Ablations}
\label{app:component_ablations}

Target grounding supervision produces the largest state-recovery benefit among the tested ablations (Table~\ref{tab:ablation_details}). Removing it reduces next-state accuracy by 7.56 percentage points, compared with a 1.16-point decrease in action macro-F1. The gap localizes a major source of execution error: a policy can retain the correct operation while losing the record to which it should apply. Grounding supervision connects action selection to the intended change in memory.

The remaining ablations reveal different sensitivities. Removing reliability supervision produces the highest pollution, linking evidence assessment to control over admission to Accepted. Removing counterfactual supervision causes the largest losses in action macro-F1 and conflict preservation, supporting supervision through the complete consequences of candidate updates. Removing multi-ledger context weakens all four metrics, showing that Pending and History/Rejected jointly provide useful context for subsequent decisions.

\section{Hyperparameter Settings}
\label{app:hyperparameters}

\paragraph{Optimization and model selection.}
We use AdamW with a learning rate of $3\times10^{-5}$.
The mini-batch size is four, with gradients accumulated
over eight steps for an effective batch size of 32.
We clip the gradient norm to 1.0 and apply action label
smoothing of 0.03. Training runs for at most 500 epochs,
with early-stopping patience of eight epochs and a
minimum improvement of $10^{-4}$. We select checkpoints
by validation five-action macro-F1, breaking ties by
next-state accuracy. Experiments use a single fixed
random seed.

\paragraph{Input budgets.}
Incoming statements and individual memory records
are truncated to 1,024 and 256 tokens, respectively.
The predictor receives at most 16 memory records
per update.

\paragraph{Loss weights and counterfactual scoring.}
We set the auxiliary loss weights to
$\lambda_{\xi}=\lambda_r=\lambda_{\Delta}=0.2$
and the counterfactual temperature to $\tau=0.2$.
The state-scoring weights for Accepted, Pending,
and History/Rejected are $0.55$, $0.25$, and $0.20$,
respectively.

\end{document}